\documentclass{article}

\PassOptionsToPackage{numbers, compress}{natbib}

\usepackage[preprint]{neurips_2021}

\usepackage[utf8]{inputenc} 
\usepackage[T1]{fontenc}    
\usepackage{hyperref}       
\usepackage{url}            
\usepackage{booktabs}       
\usepackage{amsfonts}       
\usepackage{nicefrac}       
\usepackage{microtype}      
\usepackage{xcolor}         
\usepackage{tcolorbox}
\usepackage{wrapfig}
\usepackage{multirow}

\usepackage{pifont}
\newcommand{\cmark}{\ding{51}}%
\newcommand{\xmark}{\ding{55}}%

\usepackage[inline]{enumitem}

\usepackage{caption}
\usepackage{subcaption}
\usepackage{tikz}
\usetikzlibrary{math} 
\usepackage{bm,amsmath,amsthm,amssymb}

\makeatletter
\newtheorem*{rep@theorem}{\rep@title}
\newcommand{\newreptheorem}[2]{%
\newenvironment{rep#1}[1]{%
 \def\rep@title{#2 \ref{##1}}%
 \begin{rep@theorem}}%
 {\end{rep@theorem}}}

\newtheorem*{rep@corollary}{\rep@title}
\newcommand{\newrepcorollary}[2]{%
\newenvironment{rep#1}[1]{%
 \def\rep@title{#2 \ref{##1}}%
 \begin{rep@corollary}}%
 {\end{rep@corollary}}}
 
\newtheorem*{rep@proposition}{\rep@title}
\newcommand{\newrepproposition}[2]{%
\newenvironment{rep#1}[1]{%
 \def\rep@title{#2 \ref{##1}}%
 \begin{rep@proposition}}%
 {\end{rep@proposition}}}
\makeatother

\newenvironment{psketch}{%
  \proof
}{\endproof}

\theoremstyle{definition}
\newtheorem{definition}{Definition}[section]

\theoremstyle{remark}
\newtheorem*{remark}{Remark}

\theoremstyle{plain}
\newtheorem{theorem}{Theorem}[section]
\newtheorem{proposition}[theorem]{Proposition}

\newtheorem{lemma}[theorem]{Lemma}
\newtheorem{corollary}{Corollary}[theorem]

\newreptheorem{theorem}{Theorem}
\newrepcorollary{corollary}{Corollary}
\newrepproposition{proposition}{Proposition}

\usepackage{todonotes}

\newcommand{\NEEDCITE}[1]{{\color{red} [CITE]}}

\DeclareMathOperator{\argmax}{arg\,max}

\let\Re\relax
\DeclareMathOperator{\Re}{\mathbb{R}}

\usepackage{mathtools}

\title{Counterfactual Explanations for\\Arbitrary Regression Models}

\author{%
  Thomas Spooner\\
  J.\ P.\ Morgan AI Research\\
  \texttt{thomas.spooner@jpmorgan.com} \\
  \And
  Danial Dervovic\\
  J.\ P.\ Morgan AI Research\\
  \texttt{danial.dervovic@jpmorgan.com} \\
  \And
  Jason Long\\
  J.\ P.\ Morgan AI Research\\
  \texttt{jason.x.long@jpmorgan.com} \\
  \And
  Jon Shepard\\
  J.\ P.\ Morgan AI Research\\
  \texttt{jon.shepard@jpmorgan.com} \\
  \AND
  Jiahao Chen\\
  J.\ P.\ Morgan AI Research\\
  \texttt{jiahao.chen@jpmorgan.com} \\
  \And
  Daniele Magazzeni\\
  J.\ P.\ Morgan AI Research\\
  \texttt{daniele.magazzeni@jpmorgan.com} \\
}

\begin{document}

\maketitle


\begin{abstract}
We present a new method for counterfactual explanations (CFEs)
based on Bayesian optimisation
that applies to both classification and regression models.
Our method is a globally convergent search
algorithm with support for arbitrary regression models and constraints like
feature sparsity and actionable recourse,
and furthermore can answer multiple counterfactual questions in parallel
while learning from previous queries.
We formulate CFE search for regression models in a rigorous mathematical framework
using differentiable potentials, which resolves 
robustness issues in threshold-based objectives.
We prove that in this framework,
\begin{enumerate*}[label=(\alph*)]
\item verifying the existence of counterfactuals is
    \textsf{NP}-complete; and
\item that finding instances using such potentials is
    \textsf{CLS}-complete.
\end{enumerate*}
We describe a unified algorithm for CFEs using a specialised acquisition function
that composes both expected improvement and an exponential-polynomial (EP)
family with desirable properties.
Our evaluation on real-world benchmark domains
demonstrate high sample-efficiency and precision.
\end{abstract}
















\section{Introduction}\label{sec:intro}
Counterfactual explanations (CFEs) have garnered attention in the explainable AI (XAI) literature~\citep{Miller2019} as a tool for inspecting the outputs of machine learning models \citep{Verma2020,Confalonieri2020,Miller2019}.
In its most basic form, a CFE of a model $f$ takes an input \emph{query instance} $q$ belonging to the input data space and applies a minimal perturbation $\varepsilon$ such that the new model output, $f(q + \varepsilon)$, differs from the original output, $f(q)$, in some desired way.
The ability to inspect local changes in a model's output allows one to understand a decision boundary in more detail, diagnose issues with robustness or to suggest ways consumers of a model's output can improve their own model-dependent outcomes.
The latter idea is known as \emph{actionable recourse} \citep{Karimi2021, Ustun2019}, and is important for high-stakes decisions such as credit decisioning and healthcare. The European Union now legally stipulates that any individual subject to (semi)-autonomous decision making has the \emph{right} to an explanation~\citep{voigt2017eu,Doshi-Velez2017}, and similar rights to explanation have existed in the US in domains such as credit decisioning \citep{Chen2018}.

Existing work on CFEs focus on classification models~\citep{Barocas2020}.
Since the model outputs are discrete, the notion of when a model's output changes is unambiguous.
In contrast, regression models have continuous outputs which
can change under arbitrarily small perturbations.
Applying the usual notion of counterfactuals (CFs) to explain changes to regression models
is hence prone to CFEs that are essentially uninformative 
by dint of their being too similar to the query instance.
The na\"ive solution to this problem, by requiring 
a minimal threshold of change on the dependent variables,
turns out to be
very sensitive to the chosen threshold, as we describe in \autoref{sec:cfx_search}.
This phenomenon suggests that there are questions around counterfactual specification that have yet to be addressed.

\paragraph{Our contributions.}
We present a new method for computing CFEs
for regression models that is based on Bayesian optimisation.
Our method is principled, yet flexible:
it requires only black box access to the model,
and supports arbitrary nonlinear
constraints on feature perturbations. Our specific contributions are to:
\begin{enumerate}
    \item Introduce a rigorous mathematical framework for specifying
        counterfactual search problems with regression models based on
        differentiable potentials and, in so doing, resolve the known
        robustness issues of threshold-based objectives.
    \item Prove that, under this formalism, finding an optimal CF is \textsf{CLS}-complete~\citep{fearnley:2021:complexity}, and that
        deciding CF existence more generally is \textsf{NP}-complete.
    \item Motivate the definition of an exponential-polynomial (EP) family of
        potentials, explore their theoretical properties, and
        demonstrate their effectiveness on practical problems.
    \item Provide a unified algorithm for generating counterfactuals with
        Bayesian optimisation and a specialised acquisition function based on a
        composition of expected improvement and the EP family of potential
        functions.
\end{enumerate}
To the best of our knowledge, we are the first to use Bayesian optimisation to generate CFEs.

\subsection{Related Work}\label{sec:related_work}

Several comprehensive surveys on CFEs have recently been published
~\citep{Byrne2016,Verma2020,Stepin2021}.
CF generation techniques vary by how much we can introspect into the model.
On one extreme, some methods use only black box access to model outputs, and are model agnostic~\citep{Dandl2021,FACE}.
On the other, methods exist that require full introspection into the model's specification,
notably for tree ensembles~\citep{Tolomei2017,lucic2020focus,FERNANDEZ2020196}.
Such methods are specific to a particular class of model.
Intermediate between these are methods that require gradient information,
which presumes that the model is differentiable~\citep{Pawelczyk2020b,DiCE,vanlooveren2020interpretable}.

\paragraph{Computational cost.}
In general, CFE methods that require fewer assumptions about the model have higher computational complexity.
The CF search is often formulated as a non-convex optimisation problem, although some convexified formulations do exist \citep{Artelt2019,Hammer2020},
and are computationally hard to solve optimally~\citep{MANDERS1978168,Jain_2017}.
Despite this cost, we use Bayesian optimisation (BO) for our CF search, as it is globally convergent~\citep{pmlr-v9-grunewalder10a,Srinivas_2012,JMLR:v12:bull11a} and is efficient in practice~\citep{Snoek2012,hyperopt}.
To amortize the cost incurred in CFE search,
some methods provide multiple diverse CFEs to a given query instance~\citep{DiCE,Dandl2021,FERNANDEZ2020196},
while others permit reuse of a single optimisation run to generate CFEs for multiple query instances~\citep{mahajan2020preserving}.
Our method is an example of the latter.

\paragraph{Desiderata.}
\textbf{Sparsity.} Human interpretable CFEs should involve changes to only a few features,
and is usually enforced by $\ell_1$-regularisation~\citep{lucic2020focus,FERNANDEZ2020196,Karimi2021,DiCE} or $\ell_0$-regularisation~\citep{Dandl2021}.
However, such sparsity penalties are not always appropriate, particularly when handling predictive multiplicity, when multiple trained classifiers all have similar output~\citep{Wang2019,Pawelczyk2020}.
In our method, sparsity can be induced by any $\ell_p$-regulariser as specified by the user.
\textbf{Proximity.}
Generated CFEs should be close to the manifold of observed data~\citep{Dandl2021,Pawelczyk2020b,ijcai2020-395}.
\textbf{Causality.}
CFEs should account for causal relations between features~\citep{Bottou2013,Karimi2021,mahajan2020preserving},
assuming the underlying causal structure is known~\citep{causality-pearl}.
\textbf{Actionability (recourse).}
The changes described in a CFE can be realized in a future input to the model,
which requires methods to distinguish between mutable and immutable features~\citep{ijcai2020-395,Dandl2021,Karimi2021}.
Our method permits nonlinear constraints, which allows the user to impose their own actionability properties; for instance, that age can only increase,
or that two features must change in the same direction.
\textbf{Use of categorical features.}
Categorical features must be processed into a differentiable representation before they can be used in gradient-based methods.
Methods like the standard Gumbel-softmax trick \citep{Jang2017,Maddison2017} presents scaling issues when many categories are present~\citep{wachter2018a}.
Some methods like MiVaBo~\citep{ijcai2020-365} and CoCaBO~\citep{pmlr-v119-ru20a} permit Bayesian optimisation over mixed variables,
while others use latent variables in the surrogate model \citep{Zhang2020}.
In principle, our method permits the use of any such representation of categorical variables.
In this paper, we focus on methods available in GPyOpt~\citep{gpyopt2016}, upon which we implement our algorithm.

\begin{table*}[t]
\scriptsize
    \centering
    \caption{Key properties of our counterfactual generation algorithm.}
    \label{tab:categorise-method}
    \begin{tabular}{crccc}
        \toprule\toprule
        & \textbf{Property} & \textbf{Description} & \textbf{Value} \\
        \midrule
        \emph{Assumptions} & Model Access & How much information is needed about the model? & Black Box \\
        & Model Agnostic & Is the algorithm model-independent? &  \cmark \\
        \midrule
        \emph{Optimisation} & Amortised Inference & CF generation for multiple query instances without optimising separately &  \xmark  \\
         \emph{Amortisation} & Multiple Counterfactuals & Multiple CF examples produced for a given query instance &  \cmark \\
        \midrule
        & Sparsity & Does the algorithm consider sparsity? & \cmark \\
        \emph{CF Attributes}  & Data Manifold & Are generated examples forced to be close to the data manifold? & \xmark \\
        & Causal relation & Are causal relations between features considered? & \xmark \\
        \midrule
        \emph{CF Optimisation} & Feature preference & Is feature  actionability accounted for?& \cmark \\
        \emph{Attributes} & Categorical dist. func. & Distance function for categorical features if different to continuous. & $-$ \\
        \bottomrule\bottomrule
    \end{tabular}
\end{table*}

\autoref{tab:categorise-method} describes our method using the categorisation of~\citet{Verma2020},
while noting that this categorisation excludes regressions.
At present, the only other CFE method that is black-box, model-agnostic and capable of producing multiple counterfactual example is \citet{Dandl2021}.
Like us, they pose CFE as an optimisation problem.
Also, the method applies to both regression and classification models, albeit focusing on the latter.
Our method differs in two key aspects. First, we use \textit{counterfactual potentials}
instead of thresholding, which we argue in \autoref{sec:cfx_search} is not appropriate in general.
Second, we use Bayesian optimisation, which guarantees global convergences as in \autoref{thm:ei_cfx},
instead of genetic programming, which does not provide such convergence rigorously.

\section{Counterfactual Search and its Complexity}\label{sec:cfx_search}
The generation of CFEs can be formulated as a classic
search, or satisfiability problem. Given some model $f : \mathcal{X} \to
\mathcal{Y}$, where $\mathcal{X}$ is of finite dimensionality --- but possibly infinite in cardinality --- and a query
instance $q \in \mathcal{X}$, the objective is, broadly speaking, to identify
another input $c \ne q$, $c \in\mathcal{X}$, such that the new output qualifies
as ``contrary to the fact''. To formalise this, we first introduce the notion
of a \emph{counterfactual algebra} which contains all possible subsets of a
model's range, excluding $f\!\left(q\right)$. We then identify a form of
duality between the chosen target set and the input values that could plausibly
give rise to an output in this region. These two concepts characterise the
problem of finding a CFE, regardless of the domain or
codomain of the model.

\begin{definition}[Counterfactual Algebra]\label{def:cfx_algebra}
    For a model-query pair $\left(f, q\right)$, define the \emph{counterfactual
    algebra} as the powerset
    \begin{equation}
        \mathbb{T}^f_q \doteq \mathcal{P}\!\left(\left\{f\!\left(x\right) : x \in \mathcal{X},\, f\!\left(x\right) \ne f\!\left(q\right)\right\}\right).
    \end{equation}
    Further,
    let $\widetilde{\mathbb{T}}_q^f \subset \mathbb{T}_q^f$ denote the set of subsets
    that admit polynomial-time membership circuits/oracles (see
    \autoref{def:mem_circuit}).
\end{definition}




\begin{definition}[Counterfactual Duality]\label{def:cfx_duality}
    Take a model-query pair $\left(f, q\right)$ and choose a \emph{target set}
    $\mathcal{T} \in \mathbb{T}^f_q$ to be the \emph{dual space}. The
    corresponding \emph{primal} (or, \emph{counterfactual}) space is then
    defined as the preimage of the target set under $f$,
    \begin{equation}
        \mathcal{S}^f_{\mathcal{T}} \doteq f^{-1}\!\left[\mathcal{T}\right] =
        \left\{ x \in \mathcal{X} : f\!\left(x\right) \in \mathcal{T} \right\}.
    \end{equation}
    The abbreviated notations $\mathcal{S}_\mathcal{T}$ and $\mathcal{S}$ will
    be used when $f$ and/or $\mathcal{T}$ are clear from context.
\end{definition}

The construction above makes it clear that the outcome and efficacy
of any counterfactual experiment hinges on
\begin{enumerate*}[label=(\alph*)]
    \item the nature of the domain $\mathcal{X}$ and codomain $\mathcal{Y}$;
    \item the query instance $q\in\mathcal{X}$; and
    \item the choice of target set $\mathcal{T}\in\mathbb{T}_q^f$.
\end{enumerate*}
While \autoref{def:cfx_algebra} implies that any element $x \in\mathcal{T}$ is
reachable under the model, it says nothing of how difficult it is to find such a
point. To provide intuition into this property of CF search, we define a new computational problem \textsc{CFX-Existence} below and prove in \autoref{thm:np_complete} that it is \textsf{NP}-complete. In other words, while finding a solution may be hard,
verifying CFs for a given query instance is relatively easy.

\begin{tcolorbox}[title=\textsc{CFX-Existence} (informal)]
    \textbf{Input:} A model $f : \mathcal{X} \to \mathcal{Y}$, query
    $q\in\mathcal{X}$ and target set
    $\mathcal{T}\in\widetilde{\mathbb{T}}_q^f$.
    \\[0.5em]
    \textbf{Goal:} Decide if the counterfactual set $\mathcal{S}$ is non-empty.
    \tcblower
    \begin{theorem}\label{thm:np_complete}
        \textsc{CFX-Existence} is \textsf{NP}-complete.
    \end{theorem}
    \begin{psketch}
        Polynomial verification and a reduction from \textsc{SAT}; see
        \autoref{app:complexity_results}.
    \end{psketch}
\end{tcolorbox}

\autoref{thm:np_complete} assumes that the function $f$ is efficient to compute,
which is true for many models developed in practice.
For classifiers, where the set $\mathcal{Y}$ is finite, computing set membership is trivial, and 
the convention is to express the target set for a given query point
$q$ as $\mathcal{T} = \{y \in \mathcal{Y} : y \ne
f\!\left(q\right)\}$, or some subset thereof. Since the codomain is finite, it
follows that the target set must also be finite, and thus one can construct
practical algorithms that provide candidate solutions to \textsc{CFX-Existence} under mild constraints
on the model and its domain (see \autoref{sec:related_work} and references therein).
In contrast, regression models have codomains that are subsets of an $n$-dimensional real Euclidean space,
and the definition of a ``valid'' CF is much more nuanced. First, we may
not always have an efficient algorithm for establishing whether a value is even
present in the target set. Second, it is not clear how to choose a set from the model-induced
algebra so as to obtain ``realistic'' counterfactuals.

\begin{wrapfigure}{R}{0.6\textwidth}
    \centering
    \tikzmath{\xx1 = 1.9; \xx2 = 4.0;} 
     \begin{tikzpicture}[
        declare function={ 
        g(\x) = (0.9 * ln(\x / 1.5) + 1) ;
        h(\x) = (((g(\xx2) - g(\xx1)) / pow((\xx2 - \xx1), 7) ) * pow((\x - \xx1), 7)) + g(\xx1) ;
        f(\x) = (\x <= \xx1) * g(\x) + (\x > \xx1) * h(\x) ;
        },
        ]
     \tikzmath{\qx = 0.9; \qy = g(\qx); 
                \xy1 = g(\xx1);
                \xy2 = g(\xx2);
                \epsx1 = 1.5; \epsy1 = g(\epsx1);
                \epsx2 = 2.4; \epsy2 = g(\epsx2);
                \xoffset =0.3; \epsxoffset=1.25;} 
        \draw[->] (-0.05, 0) -- (5, 0) node[right] {$x$};
        \draw[->] (0, -0.05) node[below] {} -- (0, 2.7) node[left] {$y$};
        \draw[black!50] (\qx, -0.05) node[below, black] {\small$q$} -- (\qx, 2.5);
        \draw[black!50] (-1.7, \qy) node[left, black] {\small$f(q)$} -- (\qx + \xoffset, \qy);
        \draw[black!50, dashed] (-0.1, \epsy1) node[left, black] {\small$f(q) + \varepsilon_1$} -- (\epsx1 + \xoffset, \epsy1);
        \draw[black!50] (\xx1, -0.05) node[below, black] {\small$x_1$} -- (\xx1, 2.5);
        \draw[black!50] (-1.7, \xy1) node[left, black] {\small$f(x_1)$} -- (\xx1 + \xoffset, \xy1);
        \draw[black!50] (\xx2, -0.05) node[below, black] {\small$x_2$} -- (\xx2, 2.5);
        \draw[black!50, dashed] (-0.1, \epsy2) node[left, black] {\small$f(q) + \varepsilon_2$} -- (\epsx2 + \xoffset + \epsxoffset, \epsy2);
        \draw[black!50] (-1.7, \xy2) node[left, black] {\small$f(x_2)$} -- (\xx2 + \xoffset, \xy2);
        \draw[scale=1, domain=0.5:4.2, smooth, variable=\x, blue] plot ({\x}, {f(\x)}) node[right, blue] {\small$f(x)$};
        \draw  (\xx1 + 0.25, -0.22) edge[<->] node[below] {\small$\Delta$} (\xx2 - 0.28, -0.22);

    \end{tikzpicture}
    \caption{Threshold robustness issue with CFEs: choosing $\varepsilon_2$
    over $\varepsilon_1$ yields $x_2$ rather than $x_1$ as the counterfactual.}
    \label{fig:regression_cf_toy_example}
\end{wrapfigure}
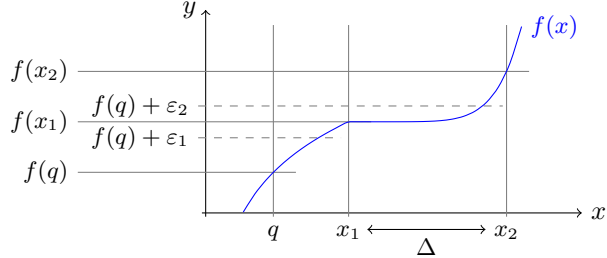

\paragraph{Setting thresholds for regression models is brittle.}
For scalar regression problems, one solution to specifying validity of CFs
is to set a threshold $\varepsilon$, so that the CF set is
$\mathcal{S} \doteq \left\{x\in\mathcal{X} :
\left\lvert f\!\left(x\right) - f\!\left(q\right) \right\rvert \geq
\varepsilon\right\}$.
This construction yields an equivalency between
instances of \textsc{CFX-Existence} for (binary) classifiers and regressors.
They are also tractable, since the targets are defined using
preorders, which admit trivial verification circuits.
However, thresholding does not distinguish between CFs in 
$\mathcal{S}$, which is a problem when the distance in $x$
far exceeds $\varepsilon$,
leading to unrealistic CFs that are far from the query point.
Furthermore, CFs defined via thresholding can be
very sensitive to
$\varepsilon$.
\autoref{fig:regression_cf_toy_example} shows an example where
the query instance $q\in\mathcal{X}$ is bounded above by $q < x_1 < x_2$,
and $f$ is monotone increasing such that
$f(q) < f(q) + \varepsilon_1 < f(x_1) < f(q) + \varepsilon_2 < f(x_2)$ for
$0 < \varepsilon_1 < \varepsilon_2$.
Define $\Delta \doteq x_2 - x_1$.
Setting $\varepsilon_1$ as the threshold yields the CF $x_1$,
whereas choosing $\varepsilon_2$ yields $x_2$ instead.
Considering the distance 
When $\Delta$ is large, $x_2$ is further from $q$ than $x_1$ is,
and thus $\varepsilon_2$ is arguably a worse threshold than $\varepsilon_1$.
as it yields CFs far from the query point.
However, there is no \textit{ex ante} way to choose between $\varepsilon_1$ and $\varepsilon_2$,
which causes this threshold robustness issue.
A key contribution our work is to formalise the notion of regression
counterfactuals in terms of \emph{potentials} instead of the direct
instantiation of the primal-dual spaces via thresholds, which we will now describe.

\subsection{Potential-Based Search}
In this work, we focus on a subset of $\mathbb{T}^f_q$ where, for
a given query $q\in\mathcal{X}$, we ascribe to each output $y \in \mathcal{Y}$ a
scalar potential which quantifies the \emph{value} associated with
candidate counterfactual points. This concept is closely related to the construction used
in potential games~\citep{monderer:1996:potential} to analyse equilibria when
agents' incentives are dictated by a single global function.
\autoref{def:potential_function} below formalises this.

\begin{definition}[Potential]\label{def:potential_function}
    A \emph{counterfactual potential
    function} is an element $\rho$ of the set
    \begin{equation}
        \mathcal{R}_{q}^f \doteq \left\{ \rho \in
        C_L^1\!\left(\mathcal{Y}, \Re\right) : \max_{x\in\mathcal{X}}
        \rho\!\left(f\!\left(x\right)\right) -
        \rho\!\left(f\!\left(q\right)\right) > 0 \right\}
    \end{equation}
    of all continuously differentiable maps between model
    outputs and the real line, where
    $(f,q)$ is a model-query pair,
    $\rho$ and $\nabla\rho$ are $L$-Lipschitz,
    and the value of $\rho$ is not maximized at $q$.
\end{definition}

Using this notion of a potential, we can now refine the counterfactual duality
of \autoref{def:cfx_duality} into \autoref{def:potential_duality} below,
which naturally characterises CFs for regression models.

\begin{definition}[Potential Duality]\label{def:potential_duality}
    For a model-query pair $\left(f, q\right)$ and potential
    $\rho\in\mathcal{R}_{q}^f$, define the
    \emph{$\varepsilon$-optimal primal-dual spaces} as
    \begin{align}
        \mathcal{T}_\rho^\varepsilon &\doteq \left\{ y = f\!\left(x\right) :
        x\in\mathcal{X},\, \rho\!\left(y\right) \geq (1 - \varepsilon)\rho^\star
        \right\} \in \mathbb{T}^f_\rho \subset \mathbb{T}^f_q,
        \label{eq:eps_target} \\
        \mathcal{S}_\rho^\varepsilon &\doteq \left\{ x\in\mathcal{X} :
        f\!\left(x\right) \in \mathcal{T}_q^\varepsilon \right\},
    \end{align}
    where $\rho^\star \doteq
    \max_{x\in\mathcal{X}}\rho\!\left(f\!\left(x\right)\right)$ and
    $0 \leq \varepsilon <1$. Let the induced sub-algebra of (potential-based)
    target sets be denoted by $\mathbb{T}_\rho^f \subset \mathbb{T}^f_q$ and
    defined as the powerset over the
    $(1-\varepsilon)\rho^\star$-superlevel sets. As in \autoref{def:cfx_algebra}, let
    $\widetilde{\mathbb{T}}_\rho^f \doteq \mathbb{T}_\rho^f \cap
    \widetilde{\mathbb{T}}_q^f$ denote the subset of efficient potential-based
    target sets.
\end{definition}

\begin{remark}
    The model and its properties will have a strong bearing on the nature of the potential-based target sets. For example, continuity and differentiability will lead to closed target sets since \autoref{def:potential_function} also stipulates that each $\rho$ be in $C^1_L$. Under certain conditions it can even be shown that
    $\mathbb{T}_\rho^f$ is a connected set in which case we can derive simple membership circuits; e.g., for boxes or bounded convex
    polytopes. This implies that it may often be practical and feasible
    to choose a target set that we know ex ante is a member of the (efficient) sub-algebra
    $\widetilde{\mathbb{T}}_\rho^f$.
\end{remark}

A counterfactual search problem that is expressible using potential duality can always be recast as an
optimisation problem, where the
goal is to find one or more points in a given counterfactual set
$\mathcal{S}_\rho^\varepsilon$.
Note, however, that \autoref{eq:eps_target} defines the target set under a global sense of optimality,
whereas existing gradient-based methods \citep{wachter2018a,DiCE}
only find locally optimal solutions.
Such methods find a CF as the limit point of a sequence $\left\{x_n\right\}_{n\in\mathbb{N}_+}$ defined by a recurrence relation of the form:
\begin{equation}\label{eq:recurrence}
    x_{n+1} \gets \Pi_{\mathcal{X}}[x_{n} + \eta\nabla
    \rho\!\left(f\!\left(x_n\right)\right)] = \Pi_{\mathcal{X}}[x_{n} + \eta\rho'(f(x_n))^\top f'(x_n)],
\end{equation}
where $\Pi_{\mathcal{X}}$ denotes an Euclidean projection onto $\mathcal{X}$,
and thus cannot apply to nondifferentiable models like boosted decision trees.
Furthermore, gradient-based methods can fail with high probability,
even if $\nabla f$ exists~\citep{Jain_2017}. Nevertheless, we show below that finding a CF using \autoref{eq:recurrence} is \textsf{CLS}-complete while finding a globally optimal CF is \textsf{CLS}-hard. These complexity results complement those of \citet{tsirtsis:2020:decisions} and also suggest that any instance of \textsc{CFX-Potential-Localopt} can be transformed in polynomial time into problems such as
\textsc{GD-Finite-Diff} \citep{fearnley:2021:complexity}. This affords us access to a powerful toolbox of methods from numerical computing and optimisation.

\begin{tcolorbox}[title=\textsc{CFX-Potential-(Localopt/Globalopt)} (informal)]
    \textbf{Input:} A differentiable model $f : \Re^n \to \Re^m$ and potential
    function $\rho \in \mathcal{R}_{q}^f$.
    \\[0.5em]
    \textbf{Goal (\textsc{Localopt}):} Find a point $c\in\mathcal{X}$ such that $\left\lvert\left\lvert c - \Pi_{\mathcal{X}}[c + \eta\nabla
    \rho\!\left(f\!\left(c\right)\right)]\ \right\rvert\right\rvert_2 \leq \delta$.
    \\[0.5em]
    \textbf{Goal (\textsc{Globalopt}):} Find an element of the counterfactual set
    $\mathcal{S}^0_\rho$.
    \tcblower
    \begin{theorem}\label{thm:cls_complete}
        \textsc{CFX-Potential-Localopt} is \textsf{CLS}-complete.
    \end{theorem}
    \begin{psketch}
        Show inclusion by reducing to \textsc{General-Continuous-Localopt}, then
        hardness follows by reducing from
        \textsc{GD-Local-Search}~\citep{fearnley:2021:complexity}; see
        \autoref{app:complexity_results}.
    \end{psketch}
    \begin{corollary}\label{corr:cls_hard}
        \textsc{CFX-Potential-Globalopt} is \textsf{CLS}-hard.
    \end{corollary}
    \begin{psketch}
        Follows directly from the injective mapping between solutions of \textsc{CFX-Potential-Globalopt} and \textsc{CFX-Potential-Localopt}; see
        \autoref{app:complexity_results}.
    \end{psketch}
\end{tcolorbox}

\subsubsection*{The Exponential-Polynomial Family}
Potentials should facilitate CFs that are neither too near nor too far
from the query, to ensure that CFs are both interpretable and actionable~\citep{Ustun2019,Pawelczyk2020}.
To this end, we introduce the exponential-polynomial (EP) family of potentials in \autoref{def:ep_family} below that have
this ``sweet spot'' property.

\begin{definition}[EP Family]\label{def:ep_family}
    For a model-query pair $(f, q)$, define the asymmetric exponential-polynomial (AEP) potentials as the functions
    \begin{equation}\label{eq:aep_potentials}
        \rho^{\textrm{AEP}_\pm}_q\!\left(y; w\right) \doteq z_q\!\left(y; w\right)_\pm^2
        \exp\!\left\{-z_q\!\left(y; w\right)_\pm^2\right\},
    \end{equation}
    whose arguments are the real value $y \in \mathcal{Y} \subseteq \Re$ and
    width parameter $w > 0$, with
    $z_q\!\left(y; w\right) \doteq
    \frac{y - f\!\left(q\right)}{w}$, $[z]_+ \doteq \max{\{z, 0\}}$, $[z]_- \doteq -\min{\{z, 0\}}$.
    The corresponding symmetric EP (SEP) potential is then
    \begin{equation}\label{eq:sep_potential}
        \rho^\textrm{SEP}_q\!\left(y; w\right) = \rho^{\textrm{AEP}_+}_q\!\left(y; w\right) + \rho^{\textrm{AEP}_-}_q\!\left(y; w\right).
    \end{equation}
\end{definition}



\begin{figure}
    \centering
    \begin{tikzpicture}
        \draw[->] (-5, 0) -- (5, 0) node[right] {$y$};
        \draw[->] (0, -0.05) node[below] {$f(q)$} -- (0, 1.5) node[right] {$\rho_q\!\left(y\right)$};
        \draw[black!50] (-1.39, -0.05) node[below, black] {\small$f(q)-w$} -- (-1.39, 1.25);
        \draw[black!50] (1.39, -0.05) node[below, black] {\small$f(q)+w$} -- (1.39, 1.25);

        \draw[scale=1.4, domain=-3:0, smooth, variable=\x, red] plot ({\x}, {2*\x*\x * exp(-\x*\x)});
        \draw[scale=1.4, domain=0:3, smooth, variable=\x, blue] plot ({\x}, {2*\x*\x * exp(-\x*\x)});
    \end{tikzpicture}

    \caption{EP potential functions for query $q$, with
    $\rho_q^{\textrm{AEP}_-}$ and $\rho_q^{\textrm{AEP}_+}$ in
    red and blue, respectively.}
    \label{fig:exponential_polynomials}
\end{figure}
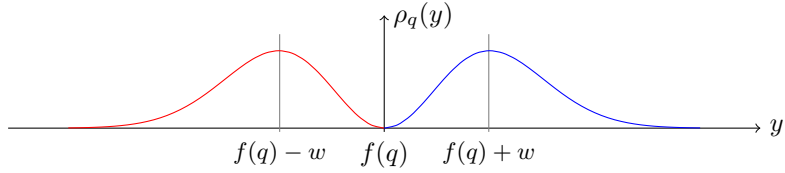

By construction, EP potentials have customisable optima,
and their shapes define nesting superlevel sets\footnote{The $\varepsilon$-optimal target sets form an annulus of outer radius $w$ and fixed inner radius under the SEP potential.}
which ensure consistent ordering of CFs.
\autoref{fig:exponential_polynomials} provides an example of EP potentials with maxima at $f(q) \pm w$,
where the signs correspond to those in \autoref{def:ep_family}.



\section{Bayesian Optimisation and the EI-CFX Algorithm}






Unlike existing work~\citep{wachter2018a,DiCE,Dandl2021}, we take a global optimisation perspective on the
problem of finding CFEs by leveraging \emph{Bayesian optimisation}~\citep{Mockus1975,Jones1998,ShahriariSWAF16},
a technique that is sample-efficient and has known convergence guarantees.
BO can be applied to non-differentiable models as it searches over
the posterior distribution of a Gaussian process (GP) surrogate,
and is thus applicable to models outside
\textsc{CFX-Potental-Globalopt}, like decision trees. However, as we show in \autoref{sec:experimental}, taking a na\"{i}ve approach to using BO in this setting can lead to very poor performance. This can be attributed to the following question: should the surrogate just predict $f$
or the entire composition $\rho \circ f$?
We argue that the former is better~\citep{astudillo:2019:bayesian},
since it makes our algorithm more effective and parallelisable by leveraging structure in the potential.

Suppose we have a model $f$ that is continuous over $\Re$
and have sampled $f\!\left(0\right) = 0$ and $f\!\left(1\right) = 2$ under a potential
$\rho\!\left(y\right) = y^2e^{-y^2}$, yielding $\rho\!\left(0\right) = 0$ and $\rho\!\left(2\right) = 4e^{-4}$.
The intermediate value theorem now implies that that there exists some $z \in \left(0, 1\right)$ with $f\!\left(z\right) = 1$ such that $(\rho\!\circ\!f)\!(z)$ attains its maximum value of $e^{-1}$.
A surrogate that models the composition $\rho \circ f$
will not be guaranteed to attain this maximum.
\autoref{fig:acq_convergence} shows corroborating numerical evidence
from a simple experiment, searching for
counterfactuals for a logistic regression model with two features.
We formulate the CF search using two different Bayesian optimisation problems,
(a) using the explicit form of our EP potential $\rho$ and modeling $f$ with the surrogate,
and (b) modeling the composition $\rho \circ f$ with the surrogate.
The former shows rapid convergence to the target set (red line),
with the acquisition function recovering the desired set with just 7 or 8 samples.
In contrast, the acquisition function of (b) does not show any meaningful
convergence even after 8 samples.

\begin{figure}
    \begin{subfigure}[t]{\textwidth}
        \centering
        \includegraphics[width=\textwidth]{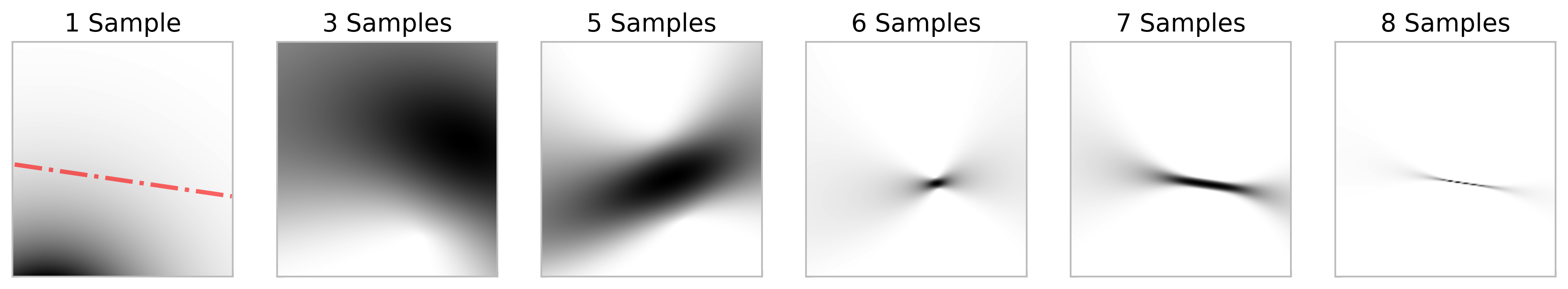}
        \caption{Exponential-Polynomial EI-CFX acquisition function (see \autoref{def:ei_cfx}).}
    \end{subfigure}
    \\
    \begin{subfigure}[t]{\textwidth}
        \centering
        \includegraphics[width=\textwidth]{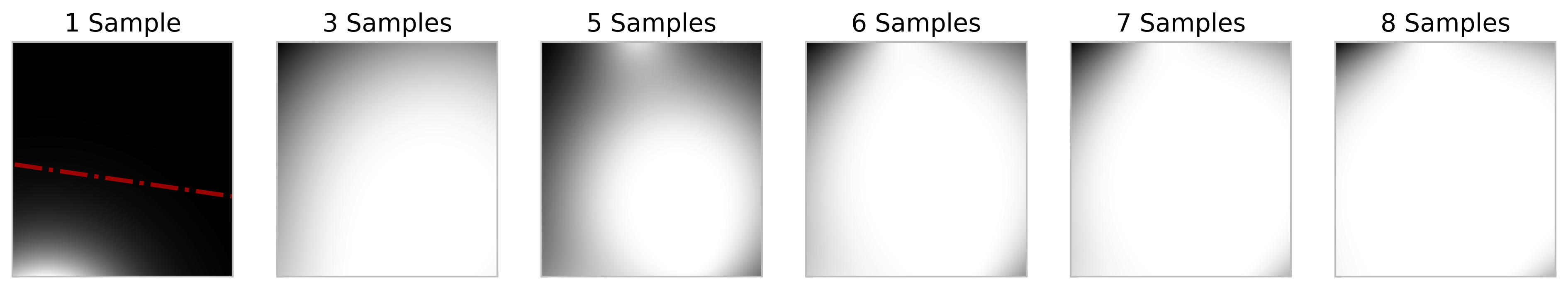}
        \caption{Traditional (black-box) EI acquisition function over composition $\rho\circ f$.}
    \end{subfigure}

    \caption{Acquisition function convergence for two features of a logistic regression model with EP potential and target of 50\% probability; see \autoref{def:ep_family}. The GP was conditioned on the same dataset in each row, darker pixels indicate higher values, and the red curve wraps the target set.}\label{fig:acq_convergence}
\end{figure}

To formalise our proposed method as a Bayesian optimisation problem we first take a scalar regression model $f : \mathcal{X} \to \Re$ and define the surrogate
$\hat f$ as being drawn from a GP prior, $\mathcal{GP}\!\left(\mu, K\right)$,
with mean function $\mu : \mathcal{X} \to \Re$ and covariance function $K :
\mathcal{X} \times \mathcal{X} \to \Re$. Given a dataset $\mathcal{D}_n
\doteq \left\{\left(x_i, f(x_i)\right)\right\}_{i\in[n]}$ we can also compute the
posterior distribution $\mathcal{GP}\!\left(\mu_n, K_n\right)$,
where the conditioned mean and covariance functions, $\mu_n : \mathcal{X} \to
\Re$ and $K_n : \mathcal{X}\times \mathcal{X} \to \Re$, can be evaluated in closed
form~\citep{gpml}.
Our novel acquisition function for performing (potential-based)
counterfactual search then follows by a refinement of the well-known expected
improvement function~\citep{Mockus1975,Jones1998}. As shown in
Proposition~\ref{prop:closed_form} and \autoref{fig:acq_convergence},
this function admits ``well-behaved,''
closed-form expressions for its value and derivative.

\begin{definition}[Expected Counterfactual Improvement]
    \label{def:ei_cfx}
    For a surrogate-potential pair $(\hat f, \rho)$, define the \emph{expected
    counterfactual improvement} as
    \begin{equation}\label{eq:ei_cfx}
        \begin{aligned}
            \textrm{EI-CFX}^\rho_n\!\left(x\right)
                &\doteq \mathbb{E}_n\!\left[\max\!\left\{0, \rho\circ\hat{f}\left(x\right) - \rho_n^\star\right\}\right]\\
                &= \frac{1}{\sqrt{K_n\!\left(x\right)}}\int_{\Re} \max\!\left\{0, \rho\!\left(y\right) - \rho_n^\star\right\}\,\phi\!\left(\frac{y - \mu_n\!\left(x\right)}{\sqrt{K_n\!\left(x\right)}}\right)\,\textrm{d}y,
        \end{aligned}
    \end{equation}
    where $\mathbb{E}_n\!\left[X\right] \doteq \mathbb{E}\!\left[X~\middle\vert~\mathcal{D}_n\right]$
    is an expectation conditioned on a set of $n$ samples
    $\mathcal{D}_n \doteq \left\{\left(x_i, f(x_i)\right)\right\}_{i\in[n]}$,
    $\rho^\star_n \doteq \max_i\left\{\rho(f(x_i)) : (x_i, f(x_i)) \in \mathcal{D}_n\right\}$
    is the maximum potential observed over these samples,
    and $\phi$ is the standard normal probability density function.
\end{definition}

\begin{proposition}\label{prop:closed_form}
    $\textrm{EI-CFX}$ and $\nabla\textrm{EI-CFX}$ are continuous functions of $\mathcal{X}$ for any $\rho \in \{\rho^{\textrm{AEP}_+}, \rho^{\textrm{AEP}_-}, \rho^{\textrm{SEP}}\}$ or $n \geq 1$.
\end{proposition}
\begin{psketch}
    Identify the 0-superlevel sets and replace the $\max$ operation with refined limits of integration. The rest follows through routine analysis of Gaussian integrals; see \autoref{app:ei_cfx-derivation}.
\end{psketch}

Under our composite structure of modeling just $f$ with the surrogate,
we can show that Bayesian optimisation using EI-CFX converges
asymptotically to a globally optimal counterfactual. In other words, assuming
that the counterfactual sets $\mathcal{S}^\varepsilon_\rho$ are non-empty for
all $\varepsilon \geq 0$, our proposed algorithm is guaranteed to find a point
$c\in\mathcal{S}^0_q$ in the limit of infinitely many observations, $n\rightarrow\infty$. This
result, which we state informally below, establishes our algorithm as the first
method for finding CFEs for regression models with
global convergence guarantees. This proof holds even for
non-differentiable models,
which is a key advantage over past approaches~\citep{wachter2018a,DiCE,Dandl2021},
and motivates a wider adoption of Bayesian optimisation for solving
instances of \textsc{CFX-Potential-Globalopt} and related problems in this area.

\begin{theorem}\label{thm:ei_cfx}
    The EI-CFX acquisition function is asymptotically consistent.
\end{theorem}
\begin{psketch}
    Follows directly from Theorem~1 of \citet{astudillo:2019:bayesian}; see \autoref{app:consistency}.
\end{psketch}

Global convergence and improved use of information are not the only advantages
of exploiting composition structure. With this approach, we are also able
to generate multiple counterfactuals in a single pass by treating the search
as a multi-objective problem \citep{lyu:2018:batch,Dandl2021}, as the
$n$-sample dataset $\mathcal{D}_n$ is entirely
independent of the particular choice of potential. This contrasts with
the alternative approach in which the potential is treated as part of the
black-box, is modelled end-to-end by the surrogate, and leaves us with a set of
samples that pertain only to a single counterfactual question. Here we can
solve an arbitrary number of acquisition functions at each step and make use of
the wider, shared dataset to bootstrap the search process,
which can lead to improved convergence rates and fewer samples being needed.





\section{Numerical Experiments}\label{sec:experimental}

The objective of this section is to establish whether the theoretical arguments
underpinning our Bayesian optimisation algorithm (and EI-CFX function) stand to
account when applied in practice. To do this, we consider two well-known
supervised learning problems based on:
\begin{enumerate*}[label=(\alph*)]
    \item the \emph{adult income} dataset~\citep{Dua:2019}, for which we use a logistic regression model; and
    \item the \emph{New York City (NYC) taxi trip duration}
        dataset~\citep{kaggle}, for which we train a light gradient boosting
        machine (LightGBM)~\citep{ke2017lightgbm}.
\end{enumerate*}
Details on hyperparameter selection and training are given in \autoref{app:model_training}.
For all experiments, the GPyOpt~\citep{gpyopt2016} library was used for the Bayesian optimisation components,
with the underlying GP furnished with an RBF kernel and zero mean function.

\begin{figure}
    \centering
    \begin{subfigure}[t]{0.48\textwidth}
        \centering
        \includegraphics[width=0.9\textwidth]{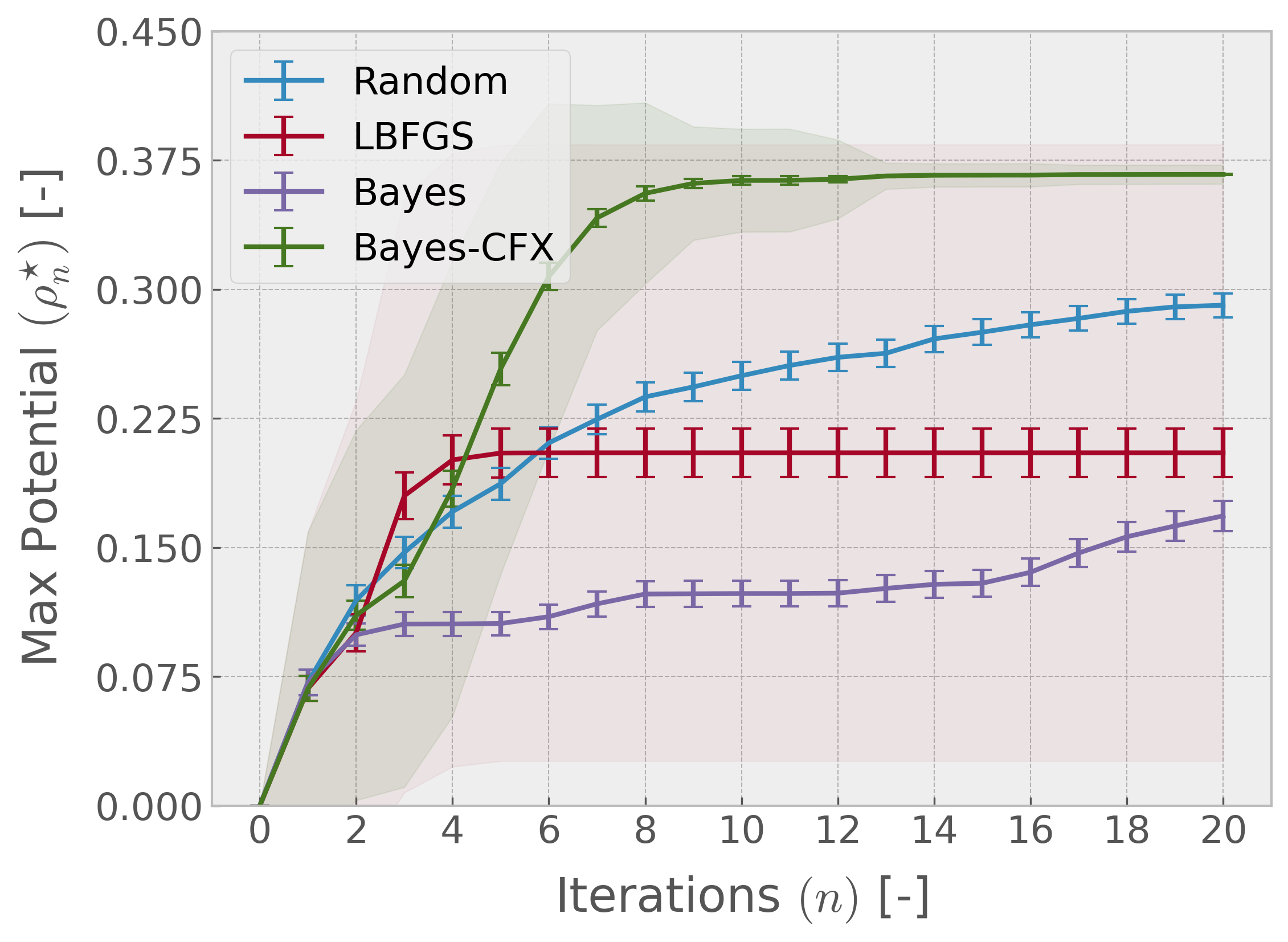}
    \caption{Average performance across all queries with the standard error of the mean.  The
        shaded regions depict the standard deviation for L-BFGS-B and
        Bayes-CFX.}\label{}
    \end{subfigure}
    \hfill
    \begin{subfigure}[b]{0.48\textwidth}
        \centering
        \begin{tabular}{rrrr}
            \toprule\toprule
            & A & CC & HPW \\
            \midrule
            Random & 18.4 & -25270 & 10.9 \\
            L-BFGS-B & 15.3 & -1060 & 12.3 \\
            Bayes & 16.3 & 54590 & 10.2 \\
            \textbf{Bayes-CFX} & \textbf{16.8} & \textbf{-24890} & \textbf{4.2} \\
            \bottomrule\bottomrule
        \end{tabular}
        \caption{Mean change to each feature across the
            counterfactuals generated for all queries
            for the features age (A), capital change (CC, 
            the sum of capital gain and capital loss), and 
            hours per week (HPW).}\label{tab:adult_residuals}
    \end{subfigure}

    \caption{A summary of the counterfactual experiment results on the adult
    income dataset.}\label{fig:adult_income}
\end{figure}

\paragraph{Adult Income.}
For the adult income dataset we posed the following broad question: ``what
would it take for the model to be class-indifferent, for young adults, given an initially confident classification?'' To assess this, we took all inputs with an age less than 30 and at least 90\% probability --- as designated by the logistic regression model --- of having a higher-income (i.e.\ $f(q) \geq 0.9$).
For each query instance $q$, we define the counterfactual potential $\rho_q$ as the AEP\textsuperscript{$-$} potential of~\autoref{def:ep_family} with the width $w$ set such that $f(q) - w \doteq 0.5$.
The CF search problem is thus formulated as the optimisation $\max_{x \in \mathcal{X}} \rho_q(f(x))$.
We then apply the following set of CF search algorithms and compare performance:
\begin{enumerate*}[label=(\alph*)]
    \item random search; 
    \item limited-memory Broyden--Fletcher--Goldfarb--Shanno quasi-Newton with bounds (L-BFGS-B) using box constraints~\citep{Byrd1995,L-BFGS-B} and finite-difference gradients, to facilitate a comparision with the other gradient-free methods;
    \item \emph{Bayes}, where we optimise the composite function $\rho_q(f(x))$ directly via Bayesian optimisation; and
    \item \emph{Bayes-CFX}, where we use the EI-CFX acquisition function of~\autoref{def:ei_cfx} with Bayesian optimisation.
    \end{enumerate*}

The results of this experiment, as illustrated in \autoref{fig:adult_income}, suggest
that our proposed method is highly sample efficient and consistent across query
instances and seeds. 
Panel (a) shows that Bayes-CFX converges after few queries to the highest potential value of the four methods.
Panel (b) shows that Bayes-CFX incurs the smallest change in the hours-per-week feature (HPW) on average, with changes in the remaining features consistent with the other methods.
The performance of L-BFGS-B was surprisingly poor, which could be due to
convergence to a local optimum or limited memory preventing such convergence entirely \cite{Byrd1995}.
Indeed, Bayes-CFX and random search both strongly support our thesis that performing
gradient-based search directly on a surrogate modeling $\rho\circ f$ is much less effective,
even when the model $f$ is differentiable. Note that random search, while certainly
sub-optimal, was also able to achieve competitive performance across all cases,
outperforming na\"{i}ve Bayesian optimisation and L-BFGS-B.

\begin{table}
    \centering
    \caption{Sample set of counterfactuals generated using Bayes-CFX on the NYC
    taxi trip duration (LightGBM) model. In each case the target deviation and
    outcome is provided, the maximum number of features allowed to change under
    an $\ell_0$ constraint, and the change in each feature.
    The query instance was for a one-person trip of 0.41~km that occurred on a Monday in May at 10AM.}\label{tab:taxi}

    \begin{tabular}{rc|cccc|r}
        \toprule\toprule
        Target && \multicolumn{4}{c|}{Change in} & Result\\
        change & $\ell_0$ Bound & Passengers & Weekday [days] & Time [hours] & Distance [km] & change \\
        \midrule
        \multirow{4}{*}{+20\%} & 1 & 0 & 0 & 0 & +2.16 & +19.6\% \\
        & 2 & 0 & +2 & 0 & +1.12 & +19.5\% \\
        & 3 & 0 & +1 & -1 & +1.18 & +19.9\% \\
        & 4 & +2 & -2 & +1 & +0.92 & +20.9\% \\
        \midrule
        +10\% & 3 & 0 & 0 & 0 & -0.31 & +10.1\% \\
        +50\% & 3 & 0 & +1 & -1 & +2.45 & +50.0\% \\
        \multirow{2}{*}{+100\%} & 3 & +2 & +2 & 0 & +3.67 & +56.7\% \\
        & $\infty$ & +2 & +2 & +1 & +3.89 & +57.9\% \\
        \bottomrule\bottomrule
    \end{tabular}
\end{table}

\paragraph{NYC Taxi Trip Durations.}
For the NYC trip duration dataset we wanted to explore the impact of sparsity (via an $\ell_0$ constraint) and satisfiability (via increasingly challenging targets) on the performance of the Bayes-CFX algorithm.
To this end, we took a random query instance --- modulo a constraint that the duration be in the lower tail of the distribution --- and posed the following question: ``what would it take to increase the trip duration by some fixed percentage?''
We allowed the search method to increase the number of passengers by up to two, perturb the weekday and hour by unit increments in the range $[-2, 2]$ and vary the distance by plus or minus one standard deviation about the query value.
Example CFEs for a representative query instance --- a one-person trip of 410m that occurred on a Monday in May at 10AM --- are shown in \autoref{tab:taxi}.

The results suggest that distance is the leading factor of trip durations: it appears in all counterfactuals regardless of the sparsity constraint. We also found that the number of passengers played a consistent role when the target deviation was very large. These are both intuitive results as they align with the natural causal model of the problem; one cannot escape the laws of physics. \autoref{tab:taxi} also highlights the value of counterfactuals for model diagnosis.
For the counterfactual targeting a fare increase of +10\%, we see that a substantial \emph{decrease} in the distance was the only necessary change,
which is unintuitive and could reflect traffic conditions, geographic locality or other confounding factors.
Additionally, the counterfactuals that are earlier in the day (time -1~hr, i.e., 9~am) tend to be more expensive,
which is understandable given that it is near the end of rush hour traffic.
Futhermore, we see also that counterfactuals that are later in the day (time +1~hr, i.e., 11~am) correlates with more passengers,
which could reflect a divide between business and social cab trips, particularly for weekend brunch (weekday -2 days and time +1 hour, i.e., Saturday at 11~am).
Finally, we note that even when there is no constraint, there are some target deviations that were not possible to satisfy; this observation was also validated via brute-force computation. Nevertheless, by accumulating a dataset of candidate counterfactuals, the Bayes-CFX algorithm was able to yield a global ``best-guess''. This is in contrast to existing (direct) gradient-based methods which are strictly local and path-dependent.

\section{Conclusions and Future Directions}
In this paper, the computational aspects of finding CFEs for regression models have been explored, the challenges around specification and formal robustness addressed, and a practical algorithm for solving such problems provided. This algorithm is shown to perform well in empirical experiments and is the first of its kind to enjoy global convergence guarantees (in the regression setting) --- a result that holds even when the model itself is non-differentiable. We argue that this motivates wider use of Bayesian optimisation in this problem domain, and suggest that future work explore this direction further. As part of the research process, we have also contributed a principled mathematical framework that lays the foundations for rigorous theoretical studies into the algebraic and geometric properties of regression counterfactuals. In particular, we note that it would be possible --- and indeed an interesting line of future work --- to refine our counterfactual algebra to the set of \emph{measurable} subsets in order to examine the probabilistic aspects of CFs. Moreover, it would also be of great interest to investigate the impact of different classes of models on the $\mathcal{S}$-$\mathcal{T}$ duality. When the model is differentiable, for example, it is clear that potential-based target sets will comprise closed and connected sets. It stands to reason that a better understanding of this and other behaviours would yield more efficient, tailored search methods.


\begin{ack}
The authors would like to acknowledge our colleague Nelson Vadori for their input
during the analysis in Proposition~\ref{prop:closed_form}.

\paragraph{Disclaimer}
This paper was prepared for informational purposes by the Artificial
Intelligence Research group of JPMorgan Chase \& Co and its affiliates (``J.P.\
Morgan''), and is not a product of the Research Department of J.P.\ Morgan.
J.P.\ Morgan makes no representation and warranty whatsoever and disclaims all
liability, for the completeness, accuracy or reliability of the information
contained herein. This document is not intended as investment research or
investment advice, or a recommendation, offer or solicitation for the purchase
or sale of any security, financial instrument, financial product or service, or
to be used in any way for evaluating the merits of participating in any
transaction, and shall not constitute a solicitation under any jurisdiction or
to any person, if such solicitation under such jurisdiction or to such person
would be unlawful.

\copyright{} 2021 JPMorgan Chase \& Co. All rights reserved.
\end{ack}

\bibliographystyle{plainnat}
\bibliography{ref}

\newpage\appendix
\section{Complexity Results}\label{app:complexity_results}
In this section we outline our proofs for the complexity results included in
the paper. We will restrict our attention only to \emph{well-behaved} circuits
that may be evaluated in polynomial-time with respect to their input and output
spaces. In particular, we leverage the definition of
\citet{fearnley:2021:complexity} which states that an arithmetic circuit is
well-behaved if, on any directed path that leads to an output, there are at
most $\log\!\left(\textrm{size}\!\left(f\right)\right)$ \emph{true}
multiplication gates. A true multiplication gate is one where both inputs
are non-constant nodes of the circuit. Even allowing an unrestricted number of
constant multiplications, it can be shown that
\begin{enumerate*}[label=(\alph*)]
    \item we can check that a given arithmetic circuit is well-behaved in
        polynomial time; and
    \item that such a circuit may be evaluated efficiently.
\end{enumerate*}
For a proof of this claim we refer the reader to Lemma~3.3 of
\citet{fearnley:2021:complexity} and their subsequent discussion on the
robustness of this formulation.

\subsection{\textsc{CFX-Existence}}
The first problem that was presented in the paper asks the question of whether
the primal space associated with a model-query-target set triple $\left(f, q,
\mathcal{T}_q\right)$ is or isn't non-empty. Note that since $\mathcal{T}_q \in
\mathbb{T}_q$, by construction we have that all elements in the dual space are
reachable, but we do not know whether the corresponding primal space
$\mathcal{S}_q$ has entries (see \autoref{def:cfx_duality}). For example, if
the model output $f\!\left(q\right)$ is only achieved at the query point
itself, then by definition there is no set in the algebra $\mathbb{T}_q$ such
that the primal is non-empty. Of course, we do not know this ex ante, and
finding a solution may require enumeration of all sets in the algebra in the
worst case! On the other hand, taking any query $q\in\mathcal{X}$, we can
always choose an empty dual space $\varnothing \in \mathbb{T}_q$ for which the
assertion is satisfied immediately. As we show below, this problem can be shown
to be \textsf{NP}-complete. To do so, we first define a \emph{membership
circuit} which is used to constrain the model domain in a principled way;
recall that the existence of such a circuit is the defining property of a
target set $\mathcal{T}$ that is an element of the efficient sub-algebra $\widetilde{\mathbb{T}}_q$.

\begin{definition}[Membership Circuit]\label{def:mem_circuit}
    Let $\mathcal{A}$ be an arbitrary topological space. A \emph{membership circuit}, $M_{\mathcal{T}} : \mathcal{A} \to \left\{0, 1\right\}$, with respect to a subset $\mathcal{T} \subseteq \mathcal{A}$ is then defined such that
    \begin{equation}
        M_{\mathcal{T}}\!\left(a\right) = \begin{cases}
            1 &\quad\textrm{if } a \in \mathcal{T}, \\
            0 &\quad\textrm{otherwise}.
        \end{cases}
    \end{equation}
\end{definition}

\begin{tcolorbox}[title=\textsc{CFX-Existence}]
    \textbf{Input:}
    \begin{itemize}[leftmargin=2em]
        \item A well-behaved arithmetic (model) circuit $f : \mathcal{X} \to \mathcal{Y}$,
        \item query instance $q\in\mathcal{X}$,
        \item and dual set $\mathcal{T}\in\widetilde{\mathbb{T}}_q^f$ with well-behaved membership circuit $M_\mathcal{T}$.
    \end{itemize}

    \textbf{Goal:} Decide if the counterfactual (primal) set $\mathcal{S}$ is non-empty.
\end{tcolorbox}

\begin{reptheorem}{thm:np_complete}
    \textsc{CFX-Existence} is \textsf{NP}-complete.
\end{reptheorem}
\begin{proof}
    We prove completeness by first showing that \textsc{CFX-Existence} admits polynomial-time solution verification, and then identifying a reduction from \textsc{SAT} which implies hardness.

    \paragraph{Inclusion.} We demonstrate a certificate of polynomial length that
    may be verified in polynomial time. Note that any proposed counterfactual
    point $c \in \mathcal{X}$ acts as a witness to verify that $\mathcal{S}$
    is non-empty. The certificate requirement is thus satisfied by construction since $\mathcal{X}$ is finite dimensional. Then, since the assertion that $f\!\left(c\right) \in \mathcal{T}$ may be computed in
    polynomial-time using the membership circuit, it follows that we have recovered the verifier definition
    of \textsf{NP}.

    \paragraph{Hardness.} We reduce from \textsc{SAT}. Let $\phi :
    \mathcal{X} \to \{0, 1\}$ denote a boolean formula which evaluates to
    \textsc{TRUE} (i.e.\ 1) iff $f\!\left(x\right) \in \mathcal{T}$, and
    \textsc{FALSE} (i.e.\ 0) otherwise; note that this operation can be
    performed in polynomial-time using the membership circuit. If the
    set $\mathcal{T}_q$ is empty then the formula $\phi$ is unsatisfiable,
    whence deciding if $\mathcal{S}$ is empty or not is \textsf{NP}-hard.
\end{proof}

\subsection{\textsc{CFX-Potential-Localopt}}
The second problem presented in the paper was concerned with \emph{finding} a counterfactual point that is locally optimal with respect to a given potential. The definition, which we state formally below, is closely related to the \textsc{GD-Local-Search} problem of \citet{fearnley:2021:complexity}. In particular, \textsc{CFX-Potential-Localopt} can be seen as a constrained variant in which we stipulate a promise that the function $\rho$ is a valid counterfactual potential; i.e.\ does not attain a global optimum at the query instance. In general, establishing whether this promise is held is hard since it amounts to a global optimisation problem in and of itself. As such, we rely on an explicit promise since it is unclear how one might define an efficient violation witness.\footnote{Note that, if we could design such a witness, a violation solution would also satisfy the requirements for our hardness reduction.} As we show in \autoref{thm:cls_complete} below, this does not change the complexity of the problem since the solution criterion is otherwise unchanged, though we acknowledge that this renders the reduction slightly less natural. Overall, the conclusions align with the results of \citet{fearnley:2021:complexity}, which suggest that the class \textsf{CLS} is robust to promise variants of the otherwise complete (total) problems considered. 

\begin{tcolorbox}[title=\textsc{CFX-Potential-Localopt}]
    \textbf{Input:}
    \begin{itemize}[leftmargin=2em]
        \item Bounded non-empty domain $\mathcal{X} = \left\{x \in \Re^n : Ax \leq b\right\}$ for $\left(A, b\right) \in \Re^{m\times n}\times\Re^m$,
        \item well-behaved arithmetic (model) circuits $f : \Re^n \to \Re^d$ and $\nabla f : \Re^n \to \Re^{d \times n}$,
        \item well-behaved arithmetic (potential) circuits $\rho : \Re^d \to \Re$ and $\nabla \rho : \Re^d \to \Re^d$,
        \item Lipschitz constant $L > 0$, step size $\eta > 0$, and tolerance $\delta > 0$.
    \end{itemize}

    \textbf{Promise:} The model circuit $\rho$ is a valid counterfactual potential (see \autoref{def:potential_function}).
    \vspace{0.25em}

    \textbf{Goal:} Find a point $c\in\mathcal{X}$ such that $\left\lvert\left\lvert c - \Pi_{\mathcal{X}}[c + \eta\nabla
    \rho\!\left(f\!\left(c\right)\right)]\ \right\rvert\right\rvert_2 \leq \delta$.
    \\[0.5em]
    Alternatively, we accept one of the following violation cases as a solution:
    \begin{itemize}
        \item One of $f$, $\nabla f$, $\rho$ or $\nabla\rho$ is not $L$-Lipschitz.
        \item $\nabla f$ is not the gradient of $f$.
        \item $\nabla \rho$ is not the gradient of $\rho$.
    \end{itemize}
\end{tcolorbox}

\begin{reptheorem}{thm:cls_complete}
    \textsc{CFX-Potential-Localopt} is \textsf{CLS}-complete.
\end{reptheorem}
\begin{proof}
    We prove completeness by first showing inclusion in \textsf{CLS}, and then identifying a reduction from \textsc{GD-Local-Search} which implies hardness.

    \paragraph{Inclusion.}
    Observe that \textsc{CFX-Potential-Localopt} reduces immediately to \textsc{General-Continuous-Localopt} by instantiating $p\!\left(x\right) \doteq \left(\rho\circ f\right)\!\left(x\right)$ and $g\!\left(x\right) \doteq x + \eta \nabla \left[\rho\!\left(f\!\left(x\right)\right)\right]$, where the gradient $\nabla\!\left[\rho\!\left(f\!\left(x\right)\right)\right] = \left[\nabla\rho\right]\!\left(f\!\left(x\right)\right)^\top \left[\nabla f\right]\!\left(x\right)$ follows from the chain rule. These quantities can be computed in polynomial-time since the operations comprise only of well-behaved arithmetic circuits and an $(1\times d)$-by-$(d\times n)$ vector-matrix product, where $d, n < \infty$. As shown by \citet{fearnley:2021:complexity} in Proposition~5.6, this implies inclusion in \textsf{CLS}. Furthermore, this result holds regardless of whether the potential promise is kept or not, since the remaining violation conditions are identical in both problems, and \textsc{General-Continuous-Localopt} otherwise has the same solution set.

    \paragraph{Hardness.} To show hardness, we need only prove that \textsc{GD-Local-Search} over the domain $\mathcal{X} \doteq \left[0, 1\right]^2$ can be reduced to \textsc{CFX-Potential-Localopt}; see Theorem~5.1 of \citet{fearnley:2021:complexity}. For this, consider an instance of \textsc{GD-Local-Search}, $(\delta, \eta, g, \nabla g, L)$, where the objective is to find a stationary point of $g$. In order to reduce to \textsc{CFX-Potential-Localopt}, we must identify an instance $(\delta, \eta, f, \nabla f, \rho, \nabla\rho, L)$ in which there is equivalence in solutions; i.e.\ an algorithm for solving \textsc{CFX-Potential-Localopt} can also be used to solve \textsc{GD-Local-Search}. This can be achieved by instantiating the potential as $\rho\!\left(y\right) \doteq y$ such that $\nabla\rho\!\left(y\right) = 1$, and model function as $f\!\left(x\right) \doteq -g\!\left(x\right)$ such that $\nabla f = -\nabla g$. This is promise-breaking in the sense that $\rho$ is not guaranteed to be a valid potential for any given $f$. However, the solutions to \textsc{CFX-Potential-Localopt} are all valid for \textsc{GD-Local-Search} even if the promise was indeed broken. With this construction we ensure that the Lipschitz constant is unchanged between the two problem instances, and thus the violation solutions are immediately preserved. We thus have a polynomial-time reduction and the proof is complete.
\end{proof}

\subsection{\textsc{CFX-Potential-Globalopt}}
\begin{tcolorbox}[title=\textsc{CFX-Potential-Globalopt}]
    \textbf{Input:}
    \begin{itemize}[leftmargin=2em]
        \item Bounded non-empty domain $\mathcal{X} = \left\{x \in \Re^n : Ax \leq b\right\}$ for $\left(A, b\right) \in \Re^{m\times n}\times\Re^m$,
        \item well-behaved arithmetic (model) circuits $f : \Re^n \to \Re^d$ and $\nabla f : \Re^n \to \Re^{d \times n}$,
        \item well-behaved arithmetic (potential) circuits $\rho : \Re^d \to \Re$ and $\nabla \rho : \Re^d \to \Re^d$,
        \item Lipschitz constant $L > 0$ and tolerance $\varepsilon > 0$.
    \end{itemize}

    \textbf{Promise:} The model circuit $\rho$ is a valid counterfactual (see \autoref{def:potential_function}).
    \vspace{0.25em}

    \textbf{Goal:} Find an element of the counterfactual set $\mathcal{S}^\varepsilon_\rho$.
    \\[0.5em]
    Alternatively, we accept one of the following violation cases as a solution:
    \begin{itemize}
        \item One of $f$, $\nabla f$, $\rho$ or $\nabla\rho$ is not $L$-Lipschitz.
        \item $\nabla f$ is not the gradient of $f$.
        \item $\nabla \rho$ is not the gradient of $\rho$.
    \end{itemize}
\end{tcolorbox}

\begin{repcorollary}{corr:cls_hard}
    \textsc{CFX-Potential-Globalopt} is \textsf{CLS}-hard.
\end{repcorollary}
\begin{proof}
    The proof follows using the same logic as that for \autoref{thm:cls_complete}, but instantiating \textsc{CFX-Potential-Globalopt} with the tolerance $\varepsilon \doteq 0$. This ensures that the only solutions are globally optimal points which will also be valid for \textsc{GD-Local-Search} since the gradient at any of the corresponding inputs must be zero. This concludes the proof.
\end{proof}

\section{EI-CFX Acquisition Function}\label{app:ei_cfx-derivation}
In this section we will go over the proof of Proposition~\ref{prop:closed_form} and examine the derivation of the closed-form expression for the EI-CFX acquisition function. Recall that, for a given model-potential pair $\left(f, \rho\right)$, the EI-CFX function is given by the Gaussian integral
\begin{displaymath}
    \textrm{EI-CFX}^\rho_n\!\left(x\right) = \frac{1}{\sqrt{K_n\!\left(x\right)}}\int_{\Re} \max\!\left\{0, \rho\!\left(y\right) - \rho_n^\star\right\}\,\phi\!\left(\frac{y - \mu_n\!\left(x\right)}{\sqrt{K_n\!\left(x\right)}}\right)\,\textrm{d}y.
\end{displaymath}
As will be shown below, this function can be evaluated efficiently and resolves to a continuous function with (relatively) well-behaved derivative.

\begin{repproposition}{prop:closed_form}
    $\textrm{EI-CFX}$ and $\nabla\textrm{EI-CFX}$ are continuous functions of $\mathcal{X}$ for any $\rho \in \{\rho^{\textrm{AEP}_+}, \rho^{\textrm{AEP}_-}, \rho^{\textrm{SEP}}\}$ or $n \geq 1$.
\end{repproposition}
\begin{proof}
    For clarity of exposition, fix the iteration number $n \geq 1$, take the symmetric EP potential $\rho^{\textrm{SEP}}$ for a query $q \in \mathcal{X}$, and define the shorthand notation $\alpha\!\left(x\right) \doteq \textrm{EI-CFX}^\rho_n\!\left(x\right)$. We will show that the analysis for $\rho^{\textrm{SEP}}$ generalises that of $\rho^{\textrm{AEP}_\pm}$ and is independent of the particular value of $n$.

    To begin, we perform a change of measure such that $\alpha\!\left(x\right)$ is expressed in terms of the standard Normal distribution:
    \begin{equation}
        \alpha\!\left(x\right) = \int_{\Re} \max\!\left\{0, \rho_q\!\left(\mu_x + \sigma_x z\right) - \rho^\star_n\right\} \, \phi\!\left(z\right) \, \textrm{d}z,
    \end{equation}
    where $\phi\!\left(\cdot\right)$ is the PDF of the standard Normal distribution with $y \doteq \mu_x + \sigma_{x} z$ and $z \sim \mathcal{N}\!\left(0, 1\right)$. The mean and standard deviation of the GP after the $n$ iterations have also been simplified to $\mu_x \doteq \mu_n\!\left(x\right)$ and $\sigma_x \doteq \sqrt{K_n\!\left(x, x\right)}$, respectively.

    Now, to remove the $\max$ non-linearity observe that, for $0 < \rho^\star_n \leq 1/e$, the SEP potential has exactly four real roots. Further, when $\rho^\star_n = 1/e$, these four values degenerate to two unique solutions that are also minima of the function. This means that we can replace the $\max$ operation with limits of integration that restrict the summation only to the positive parts. To formalise this, we provide the following auxiliary lemma:
    \begin{lemma}\label{lem:lambert_w}
        Let $c \in \left[-1/e, 0\right)$ denote a constant, then the four real solutions to the equation $-x^2 e^{-x^2} = c$ are given by $x = \pm i\sqrt{W_{-1}\!\left(c\right)}$ and $x = \pm i\sqrt{W_0\!\left(c\right)}$, where $W_k\!\left(x\right)$ denotes the $k$\textsuperscript{th} branch of Lambert's $W$-function~\cite{corless:1996:lambertw}. When $c = -1/e$, then $W_{-1}\!\left(c\right) = W_0\!\left(c\right)$ and the two remaining solutions are $x = \pm i\sqrt{W_{-1}\!\left(c\right)}$.
    \end{lemma}
    \begin{proof}
         The original equation implies that $-x^2 = W_k\!\left(c\right)$ for any $k \in \mathbb{Z}$, and thus that $x = \pm i\sqrt{W_k\!\left(c\right)}$. Since the only real branches of Lambert's $W$-function are for $k \in \{-1, 0\}$, the only real solutions must have the form stated in the claim. The degeneracy then follows from the fact that the principal branch is separated from the $k=\pm 1$ branches by a cut at $c = 1/e$~\cite{corless:1996:lambertw}.
    \end{proof}

    Lemma~\ref{lem:lambert_w} above implies that $\alpha\!\left(\cdot\right)$ may be decomposed into the following two integrals:
    \begin{equation}\label{eq:alpha_split}
        \alpha\!\left(x\right) = \underbrace{\int^{\tau^+_{-1}}_{\tau^+_0} \left[\rho\!\left(\mu_x + \sigma_x z\right) - \rho_n^\star\right] \, \phi\!\left(z\right) \, \textrm{d}z}_{\alpha^+\!\left(x\right)} + \underbrace{\int_{\tau^-_{-1}}^{\tau^-_0} \left[\rho\!\left(\mu_x + \sigma_x z\right) - \rho^\star\right] \, \phi\!\left(z\right) \, \textrm{d}z}_{\alpha^-\!\left(x\right)},
    \end{equation}
    where
    \begin{equation}
        \tau^\pm_k \doteq \frac{\pm iw\sqrt{W_k\!\left(\rho_n^\star\right)} - \tilde{\mu}_x}{\sigma_x},
    \end{equation}
    and $\tilde{\mu}_x \doteq \mu_x - \rho\!\left(q\right)$. The functions $\alpha^\pm\!\left(\cdot\right)$ have the natural interpretation of being the (re-scaled) expected counterfactual improvement with respect to the potential under negative/positive perturbations. They also correspond to applying a Gaussian kernel over the two regions for which $\rho_q\!\left(\mu_x + \sigma_x z\right) - \rho_q^\star$ is in either of the two positive quadrants, as illustrated in \autoref{fig:exponential_polynomials}.

    Focusing only on the positive case (the negative follows analogously) we can show that
    \begin{align*}
        \alpha^+\!\left(x\right)
            &= \rho_n^\star \left[\Phi\!\left(\tau^+_0\right) - \Phi\!\left(\tau^+_{-1}\right)\right] + \int^{\tau^+_{-1}}_{\tau^+_0} \rho_q\!\left(\mu_x + \sigma_x z\right) \, \phi\!\left(z\right) \, \textrm{d}z, \\
            &= \rho_n^\star \left[\Phi\!\left(\tau^+_0\right) - \Phi\!\left(\tau^+_{-1}\right)\right] + \underbrace{\frac{1}{w^2} \int^{\tau^+_{-1}}_{\tau^+_0} \left[\tilde\mu_x + \sigma_x z\right]^2 \exp{\left\{-\left[\frac{\tilde\mu_x + \sigma_x z}{w}\right]^2\right\}} \phi\!\left(z\right) \, \textrm{d}z}_{I^+\!\left(x\right)},
    \end{align*}
    where $\Phi\!\left(\cdot\right)$ denotes the standard Normal CDF. One may then solve the remaining integral by substitution, taking\footnote{Note that we drop subscripts denoting dependence for notational clarity. It should be assumed that everything depends on the input argument $x$.}
    \begin{displaymath}
        \kappa \doteq w^2 + 2\sigma_x^2, \quad a \doteq \frac{2\tilde\mu^2_x}{\kappa}, \quad b \doteq \frac{2\tilde\mu_x\sigma_x}{\kappa},
    \end{displaymath}
    such that the exponent reduces in the following way:
    \begin{align*}
        -\frac{\tilde\mu_x^2 + 2\tilde\mu_x\sigma_x\,z + \sigma_x^2\,z^2}{w^2} - \frac{z^2}{2}
            &= -\frac{1}{2w^2}\left[2\tilde\mu^2_x + 4\tilde\mu_x\sigma_x z + \left(2\sigma_x^2 + w^2\right)z^2\right], \\
            &= -\frac{\kappa}{2w^2}\left[a + 2bz + z^2\right], \\
            &= -\frac{\kappa}{2w^2}\left[\left(b + z\right)^2 + a - b^2\right].
    \end{align*}
    This means that the exponential term in $\alpha^+\!\left(\cdot\right)$ can be expressed as
    \begin{displaymath}
        \underbrace{\exp{\left\{\frac{\kappa}{2w^2}\left(b^2 - a\right)\right\}}}_{\textrm{Constant}}\,\underbrace{\exp{\!\left\{-\frac{\zeta^2}{2}\right\}}}_{\textrm{Integrand}},
    \end{displaymath}
    via a second change of variables, with $\zeta \doteq \sqrt{\frac{\kappa}{w^2}} \left(b + z\right)$. This is convenient because we have now isolated the integrand --- it remains only to evaluate a standard Gaussian integral.

    Taking $\textrm{d}\zeta = \sqrt{\frac{\kappa}{w^2}}\,\textrm{d}z$ and $z = \sqrt{\frac{w^2}{\kappa}}\,\zeta - b$, we arrive at the transformed integral
    \begin{displaymath}
        I^+ = C \int^{\tau'^+_{-1}}_{\tau'^+_0} \left[\tilde\mu_x - b\sigma_x + \sqrt{\frac{w^2}{\kappa}}\sigma_x \zeta\right]^2 \phi\!\left(\zeta\right) \, \textrm{d}\zeta,
    \end{displaymath}
    where $\tau_n'^\pm \doteq \sqrt{\frac{\kappa}{w^2}}\left(b + \tau^\pm_n\right)$ and $C \doteq \exp{\left\{\frac{\kappa}{2w^2}\left(b^2 - a\right)\right\}} \sqrt{\frac{1}{w^2\kappa}}$. This integral can then be broken down into three terms as follows:
    \begin{align*}
        \frac{1}{C} \, I^+\!\left(x\right)
            &= \int^{\tau'^+_{-1}}_{\tau'^+_0} \left[\tilde\mu_x - b\sigma_x + \sqrt{\frac{w^2}{\kappa}}\sigma_x \zeta\right]^2 \phi\!\left(\zeta\right) \, \textrm{d}\zeta, \\
            &\begin{aligned}
                &= \left[\tilde\mu_x - b\sigma_x\right]^2 \int^{\tau'^+_{-1}}_{\tau'^+_0} \phi\!\left(\zeta\right) \, \textrm{d}\zeta \\
                &~~+ 2\sqrt{\frac{w^2}{\kappa}}\sigma_x\left[\tilde\mu_x - b\sigma_x\right] \int^{\tau'^+_{-1}}_{\tau'^+_0} \zeta \phi\!\left(\zeta\right) \, \textrm{d}\zeta \\
                &~~+ \frac{w^2 \sigma_x^2}{\kappa} \int^{\tau'^+_{-1}}_{\tau'^+_0} \zeta^2 \phi\!\left(\zeta\right) \, \textrm{d}\zeta,
            \end{aligned}
    \end{align*}
    where the right hand side resolves to
    \begin{displaymath}
        \left[\left[\left(\tilde\mu_x - b\sigma_x\right)^2 + \frac{w^2 \sigma^2_x}{\kappa}\right] \Phi\!\left(\zeta\right) - 2\sqrt{\frac{w^2}{\kappa}}\sigma_x\left[\tilde\mu_x - b\sigma_x\right] \phi\!\left(\zeta\right) - \frac{w^2 \sigma^2_x}{\kappa}\zeta\,\phi\!\left(\zeta\right)\right]^{\tau'^+_{-1}}_{\tau'^+_0}.
    \end{displaymath}
    To simplify further, we can define $\gamma \doteq \tilde\mu_x - b\sigma_x$ and $\eta \doteq \frac{w\sigma_x}{\sqrt{\kappa}}$ such that
    \begin{displaymath}
        I^+\!\left(x\right) = C \left[\left(\gamma^2 + \eta^2\right) \Phi\!\left(\zeta\right) - 2\eta\gamma\phi\!\left(\zeta\right) - \eta^2\zeta\,\phi\!\left(\zeta\right)\right]^{\tau'^+_{-1}}_{\tau'^+_0},
    \end{displaymath}
    with the same logic implying that
    \begin{displaymath}
        I^-\!\left(x\right) = C \left[\left(\gamma^2 + \eta^2\right) \Phi\!\left(\zeta\right) - 2\eta\gamma\phi\!\left(\zeta\right) - \eta^2\zeta\,\phi\!\left(\zeta\right)\right]^{\tau'^-_{0}}_{\tau'^-_{-1}}.
    \end{displaymath}

    These functions are both continuous in the argument $x$ and the derivative follows from the product rule, standard Gaussian identities and properties of Lambert's $W$-function~\cite{corless:1996:lambertw} to give a closed-form expression.\footnote{We omit the explicit derivation of this expression as it's not particularly constructive. We simply note that while we implemented this exactly, one can use any robust autograd library to perform the computation with ease.} Substituting back into $\alpha^\pm\!\left(x\right)$ we clearly maintain continuity. Finally, note that
    \begin{enumerate*}[label=(\alph*)]
        \item the value of $n$ doesn't affect the result; and
        \item that the decomposition in \autoref{eq:alpha_split} aligns exactly with the distinctions between the SEP and AEP\textsuperscript{$\pm$} potentials.
    \end{enumerate*}
    It follows that the continuity property holds for all three potentials and thus the proof is complete.
\end{proof}

\section{Asymptotic Consistency and Global Optimality}\label{app:consistency}
In this final section we include a formal statement of the global convergence guarantee of our Bayesian optimisation algorithm. In particular, we show below that this result follows directly from the result proved by Theorem~1 of \citet{astudillo:2019:bayesian}.

\begin{reptheorem}{thm:ei_cfx}
    Let $\left\{x_n\right\}_{n \in \mathbb{N}}$ denote a sequence of points
    generated by the optimisation routine such that,
    for some $n_0 \in \mathbb{N}$ and all $n \geq n_0$, the iterates satisfy
    the inclusion relation
    \begin{displaymath}
        x_{n+1} \in \argmax_{x\in\mathcal{X}} \textrm{EI-CFX}_n\!\left(x\right).
    \end{displaymath}
    Then, under suitable regularity conditions, and as $n\to\infty$, we have
    that
    \begin{displaymath}
        \rho^\star_n \to \rho^\star = \max_{x\in\mathcal{X}}
        \rho\!\left(f\!\left(x\right)\right).
    \end{displaymath}
\end{reptheorem}
\begin{proof}
    The proof follows directly from Theorem~1 of \citet{astudillo:2019:bayesian} under the assumption that the covariance function of the GP satisfies the Generalised-No-Empty-Ball property (Definition~5.1 in \cite{astudillo:2019:bayesian}). This can be seen by direct instantiation of their functions $g\!\left(y\right)$ and $f\!\left(x\right)$ with the potential and model functions, respectively.
\end{proof}

\section{Model Training}\label{app:model_training}
\subsection{Adult Income}
The adult-income model comprised two stages --- a preprocessing phase and a model fitting phase --- that followed an example from the kaggle website~\cite{kaggle:adult}. We outline the two parts below.

\paragraph{Preprocessing.}
\begin{enumerate}
    \item Compress the ``Workclass'' feature such that:
        \begin{enumerate}
            \item ``Without-pay'' and ``Never-worked'' is merged into ``Unemployed.''
            \item ``State-gov'' and ``Local-gov'' is merged into ``Government.''
            \item ``Self-emp-inc'' and ``Self-emp-not-inc'' were merged into ``Self-employed''.
        \end{enumerate}
    \item Compress the ``Marital Status'' feature such that ``Married-AF-spouse,'' ``Married-civ-spouse,'' and ``Married-spouse-absent'' were merged into one value ``Married.''
    \item Group the ``Country'' feature into the following categories:
        \begin{enumerate}
            \item ``North America''
            \item ``Asia''
            \item ``South America''
            \item ``Europe''
            \item ``Other,'' which includes the spurious values ``South'' and ``?''. 
        \end{enumerate}
    \item Remove all rows that contain missing values; i.e.\ ``?'' values.
    \item Apply a standard scaling to the numerical features.
    \item Apply an ordinal encoding to the ordinal features.
    \item Apply a one-hot encoding to the categorical features.
\end{enumerate}

\paragraph{Fitting.}
The logistic regression model was then fit using an 80/20 training/testing data-split using a randomly initialised random state of 7. An $\ell_2$ regulariser was added with unit scaling. The resulting accuracy was measured at approximately 85\% on the holdout dataset.

\subsection{NYC Taxi Trip Duration}
As with the adult-income model, the model development for the NYC Taxi dataset comprised two stages --- preprocessing phase and model fitting --- that followed an example from the kaggle website~\cite{kaggle:nyc}. We outline the two parts below.

\paragraph{Preprocessing.}
\begin{enumerate}
    \item Filter the dataset for sensible values:
        \begin{enumerate}
            \item A ``trip\_duration'' of less than 5900.
            \item A non-zero ``passenger\_count.''
            \item A ``pickup\_longitude'' of greater than -100 and ``pickup\_latitude'' of less than 50.
        \end{enumerate}
    \item Apply a log transformation to the ``trip\_duration'' feature to make the distribution ``more Normal''.
    \item Apply one-hot encodings to the ``store\_and\_fwd\_flag'' and ``vendor\_id'' features.
    \item Drop the ``dropoff\_datetime'' column.
    \item Split the ``pickup\_datetime'' column into months, weeks, weekdays, hours and minutes of the day features, dropping the original column.
    \item Add Haversine distance and direction features based on the pickup and dropoff locations, and filter by those rows with ``distance'' less than 200.
    \item Compute the implied speed of the taxi and filter by those values where the new ``speed'' features was less than 30.
\end{enumerate}

\paragraph{Fitting.}
A light gradient-boosting machine model was fit on the training dataset using: a learning rate of 0.1; maximum depth of 25; 1000 leaves; feature fraction of 0.9; bagging fraction of 0.5; and a maximum bin of 1000. A random seed of 123 was used during training.

\end{document}